\documentclass[11pt]{article}
\usepackage[colorlinks=true,linkcolor=blue, citecolor=blue]{hyperref}
\usepackage[margin=1in]{geometry}
\usepackage{algpseudocode} 
\usepackage{color}
\usepackage[dvipsnames]{xcolor}
\usepackage{times}

\title{Learning to Control in Metric Space with Optimal Regret}

\author{Lin F. Yang\thanks{Princeton University, \texttt{lin.yang@princeton.edu}}, Chengzhuo Ni\thanks{Peking University, \texttt{hzxsncz@pku.edu.cn}}, Mengdi Wang\thanks{Princeton University, \texttt{mengdiw@princeton.edu }}} 

\usepackage{amsmath,color,algorithm}
\usepackage{amssymb}
\usepackage{amsthm}
\usepackage{fullpage}

\newcommand{\cM}{\mathcal{M}}
\newcommand{\cS}{\mathcal{S}}
\newcommand{\cA}{\mathcal{A}}
\newcommand{\RR}{\mathbb{R}}
\newcommand{\EE}{\mathbb{E}}
\newcommand{\wh}{\widehat}
\newcommand{\wt}{\widetilde}

\newcommand{\cK}{\mathcal{K}}
\newcommand{\reg}{\mathrm{Regret}}
\newcommand{\supp}{\mathrm{supp}}

\def\cX{\mathcal{X}}
\newcommand{\cC}{\mathcal{C}}
\newcommand{\nn}{\texttt{NN}}
\newcommand{\dist}{\texttt{dist}}
\newcommand{\funcapprox}{\texttt{FuncApprox}}

\newcommand{\cN}{\mathcal{N}}

\ifx\theorem\undefined
\newtheorem{theorem}{Theorem}
\fi
\ifx\example\undefined
\newtheorem{example}{Example}
\fi
\ifx\property\undefined
\newtheorem{property}{Property}[theorem]
\fi
\ifx\lemma\undefined
\newtheorem{lemma}[theorem]{Lemma}
\fi
\ifx\proposition\undefined
\newtheorem{proposition}{Proposition}
\fi
\ifx\remark\undefined
\newtheorem{remark}{Remark}
\fi
\ifx\corollary\undefined
\newtheorem{corollary}[theorem]{Corollary}
\fi
\ifx\definition\undefined
\newtheorem{definition}{Definition}
\fi
\ifx\conjecture\undefined
\newtheorem{conjecture}[theorem]{Conjecture}
\fi
\ifx\fact\undefined
\newtheorem{fact}[theorem]{Fact}
\fi
\ifx\claim\undefined
\newtheorem{claim}[theorem]{Claim}
\fi
\ifx\assump\undefined
\newtheorem{assump}{Assumption}
\fi
\ifx\cond\undefined
\newtheorem{cond}[theorem]{Condition}
\fi

\begin{document}

\maketitle

\begin{abstract}
We study online reinforcement learning for finite-horizon deterministic control systems with {\it arbitrary} state and action spaces. Suppose that the transition dynamics and reward function is unknown, but the state and action space is endowed with a metric that characterizes the proximity between different states and actions. We provide a surprisingly simple upper-confidence reinforcement learning algorithm that uses a function approximation oracle to estimate optimistic Q functions from experiences. We show that the regret of the algorithm after $K$ episodes is 
$O(HL(KH)^{\frac{d-1}{d}}) $ where 
$L$ is a smoothness parameter, and $d$ is the doubling dimension of the state-action space with respect to the given metric. We also establish a near-matching regret lower bound. The proposed method can be adapted to work for more structured transition systems, including the finite-state case and the case where value functions are linear combinations of features, where the method also achieve the optimal regret.
\end{abstract}

\section{Introduction}

Reinforcement learning has proved to be a powerful approach for online control of complicated systems \cite{bertsekas1995dynamic,sutton2018reinforcement}.
Given an unknown transition system with unknown rewards, we aim to learn to control the system on-the-fly by exploring available actions and receiving real-time feedback. Learning to control efficiently requires the algorithm to actively explore the problem space and dynamically update the control policy.

A major challenge with effective exploration is how to generalize past experiences to unseen states. Extensive research has focused on reinforcement learning with parametric models, for examples linear quadratic control \cite{dean2018regret}, linear model for value function approximation \cite{parr2008analysis}, and state aggregation model \cite{singh1995reinforcement}. While these parametric models would substantially reduce the complexity or regret of reinforcement learning, their practical performances are at risk of model misspecification.

In this paper, we focus on finite-horizon deterministic control systems {\it without} any parametric model. We suppose that the control system is endowed with a metric $\dist$ that characterizes the proximity between state and action pairs. We assume that the transition and reward functions are continuous with respect to $\dist$, i.e., states that are close to each other have similar values.
Proximity measures of the state-actions  have been extensively studied in the literature (see e.g. \cite{ferns2004metrics,  ortner2007pseudometrics,castro2010using, ferns2012methods,ferns2012metrics,tang2016computing} and reference therein).

Under this very general assumption, we develop a surprisingly simple upper-confidence reinforcement learning algorithm. It adaptively updates the control policy through episodes of online learning. The algorithm keeps track of an experience buffer, as is common in practical deep reinforcement learning methods \cite{mnih2013playing}. After each episode, the algorithm recomputes the Q-functions by using the updated experience buffer through a function approximation oracle. The function approximation oracle is required to find an upper-confidence Q-function that fits the known data and optimistically estimate the value of unseen states and actions. We show that the oracle can be achieved using a \emph{nearest neighbor construction}. The optimism nature of the algorithm would encourage exploration of unseen states and actions. We show that for arbitrary metric state-action space, the algorithm achieves the sublinear regret 
$$O[HL(KH)^{(d-1)/d}] $$ where $D$ is the diameter of the state-action space, $L$ is some smoothness parameter, and $d$ is the doubling dimension of the state-action space with respect to the metric. This regret is ``sublinear'' in the number of episodes played. Therefore the average number of mistakes decreases as more experiences are collected. When the state-action-space is a smooth and compact manifold, its intrinsic dimension can be substantially smaller than the observed ambient dimension.  We use an information-theoretical approach to show that this regret is optimal.

The algorithm we propose is surprisingly general and easy to implement. It uses a function approximation oracle to find ``optimistic" $Q$-functions, which are later used to control the system in the next episode. By picking suitable function approximators, we can adapt our method and analysis to more structured classes of control systems. As an example, we show that the method can be adapted to the setting where the value functions are linear combinations of features. In this setting, we show the method achieves a state-of-art regret upper bound $O(Hd)$, where $d$ is the dimension of feature space.  This regret is  also known to be optimal. 
We believe our method can be adapted to work with a broader family of function approximators, including both the classical spline methods and deep neural networks. Understanding the regret for learning to control using these function classes is for future research.

\section{Related Literatures}

Complexity and regret for reinforcement learning on stochastic systems received significant attention.  A basic setting is the Markov decision process (MDP), where the transition law at a given state $s$ and action $a$ is according to some probability distribution $p(\cdot~| s,a)$.
In the case of finite-state-action MDP without any structure knowledge, efficient reinforcement learning methods typically achieve regret that scale as $O(\sqrt{HSAT})$, where $S$ is the number of discrete states and $A$ is the number of discrete actions, $T=KH$ is number of time steps (see for examples \cite{jaksch2010near, agrawal2017optimistic,azar2017minimax,osband2016lower}). The work of \cite{jaksch2010near} provided a lower bound on the regret of $\Omega(\sqrt{HSAT})$ for $H$-horizon MDP and also regret bounds for weakly communicating infinite-horizon average reward MDP. The number of sample transitions needed to learn an approximate policy has been considered by \cite{strehl2006pac, lattimore2012pac,lattimore2014near, dann2015sample, szita2010model,kakade2003sample}. The optimal sample complexity for finding an $\epsilon$-optimal policy is $O(\frac{SA}{(1-\gamma)^3\epsilon^2})$ \cite{sidford2018near}.


In the regime of continuous-state MDP, the complexity and regret of online reinforcement learning has been explored under structured assumptions.
\cite{lattimore2013sample} studies the complexity when the true transition system belongs to a finite or compact hypothesis class, and shows sample policy that depends polynomially on the cadinality or covering number of the model class. 
Ortner and Ryabko \cite{ortner2012online} develops a model-based algorithm with  a regret bound for an algorithm that applies
to Lipschitz  MDP problems with continuous state spaces and Holder-continuous transition kernels in terms of the total variation divergence.
Pazis and Parr \cite{pazis2013pac} considers MDP with continuous
state spaces, under the assumption that  Q-functions are Lipschitz-continuous  and establishes the sample complexity bound that involves an approximate covering number.
Ok et al. \cite{ok2018exploration} studied structured MDP with a finite state space and a finite action space, including the special case of Liptschiz MDP, and provides various regret upper bounds. 

Unfortunately, existing results for Lipchitz MDP do not apply to deterministic control systems with continuity in a metric space. In the preceding works on Lipschitz MDP, the transition kernel is assumed to be continuous in the sense that $\| p(\cdot \mid s,a)- p(\cdot \mid s',a')\|_{TV} \leq L\cdot \dist ((s,a),(s',a'))$. However, this assumption almost never holds for deterministic systems. Due to the deterministic nature, the transition density is always a Dirac measure, which is discontinuous in the $\|\cdot\|_{TV}$ norm. In fact, we often have $\| p(\cdot\mid s,a)- p(\cdot \mid s',a')\|_{TV}=1$ when $(s,a)\neq (s',a')$ regardless of how close $(s,a)$ and $(s',a')$ are.

In contrast to the vast literatures on reinforcement learning for MDP,  online learning  for deterministic control has been studied by few.
Note that deterministic transition is far more common in applications like robotics and self-driving cars. 
A closely related and significant result is \cite{wen2017efficient}, which studies the complexity and regret for online learning in episodic deterministic systems. Under the assumption that the optimal $Q$ function belongs to a hypothesis class, it provides an optimistic constraint propagation method that achieves optimal regret. This result applies to many important special cases including the case of finitely many states, the cases of linear models and state aggregation and beyond. This result of  \cite{wen2017efficient} points out a significant observation that the complexity of learning to control depends on complexity of the functional class where the $Q$ functions reside in. However, the algorithm provided by \cite{wen2017efficient} is rather abstract (which is due to the generality of the method). In comparison to \cite{wen2017efficient}, our paper focuses on the setting where the only structural knowledge is the continuity with respect to a metric. 
In such a setting, \cite{wen2017efficient} would imply an \emph{infinite} regret as the Euler-dimension can be infinity in this case. 
We achieve a sublinear regret $O(K^{\frac{d-1}{d}})$ w.r.t. the number of episodes played. We show that it is optimal in this setting, and our algorithm is based on an upper-confidence function approximator which is easier to implement and generalize.

Upon finishing this paper, we are aware several recent papers working on similar settings independently. 
For instance, \cite{Song2019} and \cite{Zhu2019} study the metric space reinforcement learning problem under stochastic reward and transition. Although we obtain a similar regret bound, our contribution is focused on the more general function approximators that capture a variate of settings, in which metric space is a special case.
Another group \cite{personal_wang} is performing a comprehensive study on the linear value function setting.
Our function approximators can be potentially combined with their results to obtain a better algorithm in the linear setting.

\section{Problem Formulation}

We review the basics of Markov decision problems and the notion of regret. 

\subsection{Deterministic MDP}

Consider a deterministic finite-horizon Markov Decision Process (MDP) $\cM = \{\cS, \cA, f, r,H\}$, where $\cS$ is an {\it arbitrary} set of states, $\cA$ is an {\it arbitrary} set of actions, $f:\cS\times \cA \rightarrow \cS$ is a deterministic transition function, $H\ge 1$ is the horizon, and $r:\cS\times \cA\rightarrow [0,1]$ is a reward function.
A policy is a function $\pi:\cS\times[H]\rightarrow \cA$ ($[H]$ denotes the set of integers $\{1,\ldots,H\}$).
The optimal policy $\pi^*$ maximizes the cumulative reward in $H$ time steps, from any fixed initial state $s_0$:
\begin{align*}
\max_{\pi} \sum^H_{h=1} &r(s_h, a_h)\\
\hbox{subject to }s_{h+1} &=f(s_h,a_h),\\
	 a_h &= \pi(s_h, h), \\
	 	s_1&=s_0.
\end{align*}

Given a policy $\pi$, the value function $V^{\pi}:\cS\times [H]\rightarrow\RR$ is defined recursively as follows.
\[
\forall s \in \cS:\quad V^{\pi}_H(s) = r(s, \pi(s,H)) \]
and $\forall  h\in [H-1]:$
\[ \quad V^{\pi}_h(s) = r(s, \pi(s,h)) + V^{\pi}_{h+1} [f(s,\pi(s,h))]
\]
An optimal policy $\pi^*$  satisfies 
\[
\forall h:~V^{\pi^*}_h := V^{*}_h =\max_{\pi} V_h^{\pi} \quad\text{entrywisely}.
\]
In particular, the optimal value function $V^*$ satisfies the following Bellman equation
\[
\forall s \in \cS, h\in [H-1]:\quad V^{*}_H(s) = \max_{a\in \cA}r(s, a) \]
and
\[ \quad V^*_h(s) = \max_{a\in \cA}\{r(s, a) + V^{*}_{h+1}[f(s,a)]\}.
\] 
We also define the $Q$-functions, $Q_h^{\pi}:\cS\times\cA\rightarrow \RR$, as, $\forall h\in [H-1]:$
\begin{align*}
\quad Q_h^{\pi}(s,a) &= r(s, a) + V^{\pi}_{h+1}[f(s,a)]= r(s, a) + Q^{\pi}_{h+1}[f(s,a), \pi(f(s,a), h+1)],
\end{align*}
where $Q_H^{\pi}(s,a) = r(s,a)$.
We further denote $Q_h^{*} = Q_h^{\pi^*}$ for $h\in[H]$.
%
%
\subsection{Episodic Reinforcement Learning and Regret}
We focus on the {\it online episodic reinforcement learning} problem, in which the learning agent does not know $f$ or $r$ to begin with.
The agent repetitively controls the system for episodes of $H$ time, where each episode starts from some initial state $s_0$ that does not depend on the history.
We denote the total number of episodes played by the agent as $K\ge 1$.

Suppose that the learning agent is an algorithm $\cK$ (possibly randomized). It can observe all the state transitions and rewards generated by the system and adaptively pick the next action. We define its regret of this algorithm $\cK$ as 
\[
\reg_{\cK}(K) = \EE^{\cK}\Big[K\cdot V^*(s_0, 1) - \sum_{k=1}^{K}\sum_{h=1}^{H} r(s^{(k)}_h, a^{(k)}_h)\Big],
\]
where the action $a^{(k)}_h$ is generated by the algorithm $\cK$ at time $(k,h)$ based on the entire past history, and $\EE^{\cK}$ is taken over the randomness of the algorithm $\cK$.
In words, the regret of $\cK$ measures the difference between the total rewards collected by the algorithm and that by the optimal policy after $K$ episodes.

\section{The Basic Case of Finitely Many States and Actions}
\begin{algorithm*}[htb!]
	\caption{Upper Confidence Reinforcement Learning for Deterministic Finite MDP \label{opt-discrete}}
	\begin{algorithmic}[1]
		\State 
		\textbf{Input:} A deterministic MDP.
		\State \textbf{Initialize:} $Q^{(1)}_{h}\gets H\cdot \mathbf{1} \in \RR^{\cS\times \cA}$ for every $h$;
		\State \textbf{Initialize:} For every $(s,a)\in \cS\times \cA$, $\wh{r}(s,a) \gets 1$, $\wh{f}(s,a)\gets \texttt{NULL}, b(s,a)\gets H$;
		\For{episode $k=1, 2, \ldots, K,\ldots$}
		\For{stage $h=1, 2, \ldots, H$}
		\State Current state: $s^{(k)}_h$;
		\State Play action $a^{(k)}_h \gets \arg\max_{a\in \cA} Q^{(k)}_h(s, a)$;
		\State Observe the state transition $s^{(k)}_{h+1}\gets f(s^{(k)}_h, a^{(k)}_h)$ and obtain reward $r(s^{(k)}_h, a^{(k)}_h)$;
		\State Update: $\wh{f}(s^{(k)}_h, a^{(k)}_h)\gets s^{(k)}_{h+1}$ and $\wh{r}(s^{(k)}_h, a^{(k)}_h)\gets r(s^{(k)}_h, a^{(k)}_h)$;
		\State Update: $b(s^{(k)}_h, a^{(k)}_h)\gets 0$;
		\EndFor
		\State Obtain new value functions $Q_h^{(k+1)}$ using $\hat f,\hat r,b$ by dynamic programming:
		\begin{align*}
		\forall s \in \cS, h\in [H-1]:\quad Q^{(k+1)}_H(s, a) &\gets \wh{r}(s, a) \quad\text{and} \quad \\
		Q^{(k+1)}_h(s, a) &\gets \min\Big(H, ~\wh{r}(s, a) + \max_{a'\in\cA }Q^{(k+1)}_{h+1}[\wh{f}(s,a), a'] + b(s,a)\Big),
		\end{align*}
		\qquad where we denote $ Q^{(k+1)}_h[\texttt{NULL}, a] = 0$.
		\EndFor
	\end{algorithmic}
\end{algorithm*}

 We provide Algorithm~\ref{opt-discrete} for the case where the state space $\cS$ is a finite set of size $S$ and the action space $\cA$ is a finite set of size $A$, without assuming any structural knowledge. 
Note although \cite{wen2017efficient} has provided a regret-optimal algorithm for this setting, we provide a simpler algorithm based on upper-confidence bounds. Despite of the simplicity of this setting, we include the result to illustrate our idea, which might be of independent interest.

Algorithm~\ref{opt-discrete} 
always maintains an upper bound of the optimal value function using the past experiences.
The algorithm 
uses the value upper bound to plan the future actions.
After each episode, the value upper bound is improved based on the newly obtained data.
Since the exploration is based on the upper bound of the value function, it always encourages the exploration of un-explored actions. 
The value is improved in such a way that once a regret is paid, the algorithm is always able to gain some new information such that the same regret will not be paid again.
The guarantee of the algorithm is presented in the following theorem.

\begin{theorem}
\label{thm:opt-disc-reg}
After $K$ episodes,
the above algorithm obtains a regret bound
\[
\reg(K) \le SAH.
\]
\end{theorem}
The proof of the algorithm is presented in the appendix. 
Whenever a state-action is visited for the first time, an instant regret $H$ is paid and the algorithm ``gains confidence" by setting the confidence bound $b(s,a)$ to zero. This can happen at most $SA$ times.
Theorem \ref{thm:opt-disc-reg} matches the regret upper and lower bound in the finite case, which was proved in \cite{wen2017efficient}.
Note that our setting is slightly different from that of  \cite{wen2017efficient} (they assumed $f$ to be time-dependent). For completeness, we include a rigorous regret lower bound proof for our setting. See Theorem~\ref{thm:reg-lb-det} in the appendix.

\section{Policy Exploration In Metric Space}

Now we consider the more general case where the state-action space $\cX = \cS\times\cA$ is arbitrarily large. For example, the state in a video game can be a raw-pixel image, and the state of a robotic system can be a vector of positions, velocities and acceleration. In these problems the state space can be considered a smooth manifold in a high-dimensional ambient space. 

\subsection{Metric and Continuity}

The major challenge with reinforcement learning is to generalize past experiences to unseen states. For the sake of generality, we only assume that a proper notion of distance between states is given, which suggests that states that are closer to each other have similar values.

Suppose we have a metric\footnote{In fact, our analysis does not require the condition $\dist(x,y)=0\Leftrightarrow x=y$. Hence the metric space can be further relaxed to pseudometric space.} $\dist(\cdot,\cdot)$ over the state-action space $\cX=\cS\times\cA$, i.e.,
$\dist(x,y) =\dist(y,x)$, and $\dist$ satisfies the triangle inequality.

\begin{assump}[MDP in metric space with Lipschitz continuity]\label{assump-metric}
	Let the optimal action-value function be $Q_h^*: \cS\times\cA\rightarrow \RR$.
	Then there exist constants $L_1, L_2>0$ such that $\forall (s, a), (s', a')\in\cS\times\cA$ and $\forall h\in [H]$, $r(s,a)\in [0,1]$,
	\begin{align}
	\label{eqn:q-cont}
	|Q_h^*(s,a) - Q_h^*(s', a')| &\le L_1 \cdot \dist[(s,a), (s',a')]
	\end{align}
	and
	\begin{align}
	\max_{a''}\dist[(f(s, a), a''), (f(s', a'), a'')]&\leq L_2\cdot \dist((s, a), (s', a'))
	\end{align}
	We further denote $L=(L_2+1)\cdot L_1$ for convenience.
\end{assump}

\subsection{Optimistic Function Approximation}
To handle the curse of dimensionality of general state space, we will use a function approximator for computing optimistic Q-function from experiences. 
The function approximator needs to satisfy the following conditions.
\begin{assump}[Function Approximation Oracle]\label{assump-func}
	Let $q:\cX\rightarrow \RR$ be a function.
	Let $B:=\{(x_i,q(x_i))\}_{i=1}^N\subset \cX\times \RR$ be a set of key-value pairs generated by function $q$.
	Let $L>0$ be a parameter.
	Then there exists a function approximator, $\funcapprox$, which, on given $B$, outputs a function $\wh{q}:\cX\rightarrow \RR$ that satisfies
\begin{enumerate}
\item $\wh{q}$ is $L$-Lipschitz continuous;
\item $\forall x\in \cX:\quad \wh{q}(x)\geq q(x)$;
\item $\forall  i\in [N]:\quad \wh{q}(x_i)  =  q(x_i)$\footnote{This condition can be further relaxed to $\wh{q}(x_i)  \le  q(x_i) + \Delta$, for some error parameter $\Delta\ge 0$. In this case, the regret bound is linearly depending on $\Delta$.}.
\end{enumerate}
\end{assump}

One way to achieve the conditions required by the function approximator is to use \emph{the nearest neighbor approach}, given by
\begin{equation}\label{eq-funcnn}
\forall x\in \cX: \wh{q}(x):=  \min_{i\in[N]} \{ q(x_i) + L \cdot \dist(x,x_i)\},
\end{equation}
	where the distance regularization $ L \cdot \dist(x,x_i)$ will overestimate the value at an unseen point using its near neighbors.
\begin{lemma}
	\label{lem:q-lipschitz}
Suppose the function $q$ is $L$-Lipschitz.
Then the nearest neighbor approximator given by \eqref{eq-funcnn} is a function-approxiamtor satisfying Assumption~\ref{assump-func} with Lipschitz constant $L$. 
\end{lemma}
\begin{proof}
Firstly, we observed that, by triangle inequality,
for any $x,x'\in \cX$,
\begin{align*}
|\wh{q}(x) - \wh{q}({x'})|
&\le \max_{i}|q(x_i) + L\cdot\dist(x, x_i) - q(x_i) - L\cdot\dist(x', x_i)|\\
&\le L\cdot\max_i|\dist(x, x_i) - \dist(x', x_i)|\\
&\le L\cdot \dist(x, x').
\end{align*}
Therefore (1) of Assumption~\ref{assump-func} holds.

Secondly, since $q$ is Lipschitz continuous, we have, for all $i\in [N]$,
\[
q(x)\le q(x_i) + L\cdot \dist(x, x_i).
\]
Thus \[
q(x)\le \min_{i} q(x_i) + L\cdot \dist(x, x_i) = \wh{q}(x)
\]
and
(2) of Assumption~\ref{assump-func} holds.

We now verify (3) of Assumption~\ref{assump-func}.
For all $j\in[N]$, we have,
\[
{q}(x_{j}) \le \wh{q}(x_{j}) = \min_{i} \{q(x_i) + L\cdot \dist(x_j, x_i)\} \le q(x_j)+ L\cdot \dist(x_{j}, x_{j}) 
= q(x_j),
\]
 as desired.
\end{proof}

More generally speaking, one can construct the function approximator $\wh{q}$ by solving a regression problem. For example, suppose that $q$ is integrable with respect to a measure $\mu$ over the state-action space. Then we can find it using
\begin{align*}
&\max_{\wh{q}\in\mathcal{F}, \wh{q}\ge q} \int q d\mu(x),\quad \hbox{s.t. } \wh{q}(x_i)=q(x_i), \forall i\in[N],
\end{align*}
or for some arbitrarily small $\delta>0$,
\begin{align*}
&\max_{\wh{q}\in\mathcal{F}, \wh{q}\ge q} \sum_{i=1}^N -(q(x_i) - \wh{q}(x_i))^2 + \delta \int \wh{q} d\mu(x),
\end{align*}
where $\mathcal{F}$ is the set of $L$-Lipschitz functions with infinity norm bounded by $H$. This formulation is compatible with a broader family of function classes, where $\mathcal{F}$ can be replaced by a parametric family that is sufficient to express the unknown $q$, including spline interpolation and deep neural networks. 

\subsection{Regret-Optimal Algorithm in Metric Space}

Next we provide a regret-optimal algorithm for the metric MDP.
The algorithm does not need additional structural assumption other than the Lipschitz continuity of $Q^*$ and $f$. 
It is a combination of the UCB-type algorithm with a nearest-neighbor search.
It measures the confidence by coupling the Lipschitz constant with the distance of a newly observed state to its nearest observed state.
The algorithm is formally presented in Algorithm~\ref{alg:opt-cont}.

The algorithm keeps an experience buffer $B^{(k)}=\big\{(s^{(1)}_1, a^{(1)}_1), (s^{(2)}_2, a^{(2)}_2), \ldots, (s^{(k)}_H, a^{(k)}_H)\big\}$ that grows as new sample transitions are observed. It optimistically explores the policy space in online training using upper-estimate of Q-values. These Q-values are computed recursively by using the function approximator according to the dynamic programming principle.

 In particular, if the function approximation oracle is given by the nearest neighbor construction \eqref{eq-funcnn},
Step 11 and Step 12 of the algorithm take the form of
\begin{align}
\label{eqn:funcapprox-1}	
\wh{r}^{(k+1)}(s,a) &= \min\Big[\min_{(s',a')\in B^{(k+1)}}\Big(r(s',a')  + L_1\cdot \dist[(s,a), (s',a')]\Big), 1\Big]\nonumber\\
Q_H^{(k+1)}(s, a)&\gets  \wh{r}^{(k+1)}(s, a)\nonumber\\
		{Q}^{(k+1)}_h(s, a) &\gets \min_{(s',a')\in B^{(k+1)}}\Big[ r(s',a') + \sup_{a''\in\cA}{Q}_{h+1}^{(k+1)}({f}(s',a'), a'')+ L_1\cdot \dist[(s',a'), (s,a)]\Big] 
\end{align}
		for all  $(s,a) \in \cS\times\cA, ~h\leq H-1$. In this case, the nearest-neighbor function approximator will  prioritize exploring 
$(s,a)$'s that are farther away from the seen ones.
Note that the function approximators defined in \eqref{eqn:funcapprox-1} satisfy Assumption~\ref{assump-func} (see Lemma~\ref{lem:q-lipschitz}).

Algorithm~\ref{alg:opt-cont} provides a general framework for reinforcement learning with a function approximator in deterministic control systems. It can be adapted to work with a broad class of function approximators.


\begin{algorithm*}[htb!]
	\caption{Upper Confidence Reinforcement Learning with Function Approximator (UCRL-FA) \label{alg:opt-cont}}
	\begin{algorithmic}[1]
		\State 
		\textbf{Input:} A deterministic metric MDP.
		\State \textbf{Initialize:} Initialize $B^{(0)}\gets\emptyset$, $Q_h^{(0)}(s,a)\gets H, \wh{r}^{(0)}(s,a)\gets 1$, for all $(s,a)\in \cS\times \cA, h\in [H]$;
		\For{episode $k=1, 2, \ldots, K,\ldots$}
		\For{stage $h=1, 2, \ldots, H$}
		\State Current state: $s^{(k)}_h$;
		\State  Play action $a^{(k)}_h = \arg\max_{a\in \cA} Q^{(k)}_h(s_h^{(k)}, a)$
		\State Record the next state $s^{k}_{h+1}\gets f(s^{(k)}_h,a^{(k)}_h)$ and reward $r(s^{(k)}_h,a^{(k)}_h)$;
		\EndFor
		\State Update {\small$B^{(k+1)} \gets B^{(k)}\cup\left\{\left(s_1^{(k)}, a_1^{(k)},f(s^{(k)}_1,a^{(k)}_1),r(s^{(k)}_1,a^{(k)}_1)\right),  \ldots, \left(s_H^{(k)}, a_H^{(k)},f(s^{(k)}_H,a^{(k)}_H),r(s^{(k)}_H,a^{(k)}_H)\right)\right\}$};
		\State Now we update $Q^{(k+1)}_h$ recursively as following: 
		\State \quad We first denote a modified reward $\wh{r}^{(k+1)}$, 
		\begin{align*}
		 \wh{r}^{(k+1)} &\gets \funcapprox\Big( \{ (s,a), r(s,a)\}_{(s,a)\in B^{(k+1)}}  \Big)\\
		 Q_H^{(k+1)}& \gets  \wh{r}^{(k+1)}
		\end{align*}
		\State \label{line:qkh}\quad We then denote the value function ${Q}^{(k+1)}_h$ as,
		\begin{align*}
		{Q}^{(k+1)}_h &\gets \funcapprox\left(\left \{ (s,a),  r(s,a) + \sup_{a'\in\cA}{Q}_{h+1}^{(k+1)}({f}(s,a), a')\right\}_{(s,a)\in B^{(k+1)}} \right )
		\end{align*}
		\qquad\quad  Note that in the above step $f(s,a)$ is known since $(s, a)\in B^{(k+1)}$.
		\EndFor
	\end{algorithmic}
\end{algorithm*}

\section{Regret Analysis}

In this section, we prove the main results of this paper.

\subsection{Main Results}
For a metric space $\cX$, we denote the \emph{$\epsilon$-net},
$\cN(\epsilon)\subset \cX$, as a set such that
\[
\forall x\in \cX:\quad
\exists x'\in \cN(\epsilon),~ s.t.\quad \dist(x,x')\le\epsilon.
\]
If $\cX$ is compact, we denote $N(\epsilon)$ as the \emph{minimum} size of an $\epsilon$-net for $\cX$.
We also denote a similar concept, the \emph{$\epsilon$-packing}, $\cC(\epsilon)\subset \cX$, as a set such that
\[
\forall x,x'\in\cC(\epsilon):\quad
\dist(x,x')>\epsilon. 
\]
If $\cX$ is compact, we denote $C(\epsilon)$ as the \emph{maximum} size of an $\epsilon$-packing for $\cX$.
In general, $N(\epsilon)\le C(\epsilon)$ and are of the same order.
For a normed space (the metric is induced by a norm), we have
$C(2\epsilon)\le N(\epsilon)\le C(\epsilon)$.

Next we show that the regret till reaching $\epsilon$-optimality is upper bounded by a constant that is proportional to the size of the $\epsilon$-net.
We will show later that the regret is lower bounded by a constant proportional to the size of the $\epsilon$-packing.

\begin{theorem}[{\bf Regret till $\epsilon$-optimality}]
	\label{thm:metric-regretbound}
	Suppose we have an episodic deterministic MDP $M=(\cS, \cA, f, r, H)$ that satisfies Assumption~\ref{assump-metric}. Let $\cX = \cS\times \cA$ be a state-action space with diameter $D>0$, and  $L_1, L_2$ be parameters specified in Assumption \ref{assump-metric}.
	Suppose we use the $L_1$-continuous function approximator defined in \eqref{eqn:funcapprox-1}.
	Suppose the state-action space $\cX$ admits an $\epsilon$-cover $\cN(\epsilon)$ for any $\epsilon>0$.
	Then  
	after $T=KH$ steps, Algorithm~\ref{alg:opt-cont} obtains a regret bound
	\begin{align*}
	\reg(K)\le H|\cN(\epsilon)| + 2\epsilon L KH.
	\end{align*}
	where  $L=(L_2+1)\cdot L_1$.
\end{theorem}

Suppose $d$ is the doubling dimension of the state-action space $\cX$. The doubling dimension of a metric space is the smallest positive integer, $d$, such that every ball  can be covered by $2^d$ balls of half the radius. Then we can show the following regret bound.

\begin{theorem}[{\bf Optimal Regret for Metric Space}]
	\label{thm:corr-metric-regret}
	Suppose the state-action space is compact with diameter $D$ and has a doubling dimension $d > 0$.
	Then  after $K$ episodes, Algorithm~\ref{alg:opt-cont} with a nearest-neighbor function approximator \eqref{eqn:funcapprox-1} obtains a regret bound
	\begin{align*}
	\reg(K)&=O(DLK)^{\frac{d}{d+1}}\cdot H.
	\end{align*}
\end{theorem}
The  regret bound is sub-linear with in the number of steps $T:=KH$ and linear with respect to the smoothness constant $L$ and diameter $D$.

\paragraph{About Doubling Dimension:} The regret depends on the doubling dimension $d$. 
It is the intrinsic dimension of $\cX$ - often very small even though the observed state space has high dimensions. For example, the raw-pixel images in a video games often belong to a smooth manifold and has small intrinsic dimension.
Our Algorithm 2 uses the nearest-neighbor function approximation. It can be thought of as learning the manifold state space at the same time when solving the dynamic program. It does not need any parametric model or feature map to capture the small intrinsic dimension.
\subsection{Proofs of the Main Theorems}
To prove the above theorems, we need several core lemmas.
The following lemma shows that the approximated Q-function (in Algorithm~\ref{alg:opt-cont}) is always an upper bound of the optimal  Q-function.
Note that this lemma works for all function approximators that satisfies Assumption~\ref{assump-func}.
\begin{lemma}[Optimism]
	\label{lemma:optimism}
	Suppose Assumption~\ref{assump-func} holds for the \funcapprox~in Algorithm~\ref{alg:opt-cont}.
	Then, for any $k\in [K]$,  $(s,a)\in \cS\times \cA,  h\in[H]$, and $k\in[K]$, we have
	\begin{align*}
	  Q^*_h(s, a)\leq Q^{(k)}_h(s, a) 
	\end{align*}
\end{lemma}
\begin{proof}
	By the properties of $\funcapprox$, we have
	\[
	r \le \wh{r}^{(k)} \qquad \text{entriwisely}.
	\]
	We prove the result by induction: 
	when $h=H$, we have 
	\begin{align*}
	Q^*_H(s, a) = r(s, a) \le \wh{r}^{(k)}(s,a):= Q^{(k)}_H(s, a).
	\end{align*}
	Suppose the relation holds for $h+1$, then,  we have
	\begin{align*}
	{Q}^{(k)}_h \gets \funcapprox\big( \{ (s,a),  r(s,a)
	 &+  \sup_{a'\in\cA}{Q}_{h+1}^{(k)}({f}(s,a), a')\}_{(s,a)\in B^{(k+1)}}  \big).
	\end{align*}
	Thus we have
	\[
	{Q}^{*}_h(s,a)=r(s,a) + \sup_{a'\in\cA}{Q}_{h+1}^{*}({f}(s,a), a') \le r(s,a) + \sup_{a'\in\cA}{Q}_{h+1}^{(k)}({f}(s,a), a')
	\le {Q}^{(k)}_h(s,a).
	\]
	This completes the proof.
%
%
\end{proof}
Given a finite set $B = \{(s_i, a_i, f(s_i, a_i), r(s_i, a_i))\}$, for any $(s, a)\in\cS\times\cA$, we define the nearest neighbor operator $\nn$ and function $b^B$ as followings,
\begin{align*}
\nn(B,(s, a)) &= \arg\min_{(s', a')\in B}\dist((s, a), (s', a')),\\
b^B(s, a) &= \dist[(s, a), \nn(B,(s, a))]
\end{align*}
The next lemma shows that Algorithm~\ref{alg:opt-cont} with function approximator \eqref{eqn:funcapprox-1}  does not incur too much per-step error.
\begin{lemma}[Induction]
	\label{lem:ind}
	Suppose the \funcapprox~in Algorithm~\ref{assump-func} is \eqref{eqn:funcapprox-1}.
	Then for any $k\in [K]$ $ (s,a)\in \cS\times \cA,  h\in[H]$, and $k\in[K]$, we have,
	\begin{align*}
	 Q^{(k)}_h(s^{(k)}_h, a^{(k)}_h)
	&\le  r(s^{(k)}_{h+1}, a^{(k)}_{h+1}) + Q^{(k)}_{h+1}(s^{(k)}_{h+1}, a^{(k)}_{h+1}) + L\cdot b^{B^{(k)}}(s_h^{(k)},a_h^{(k)})
	\end{align*}
\end{lemma}
\begin{proof}
	Let 
	\[
	\Big(s^{(k)*}_h, a^{(k)*}_h\Big) 
	= \arg\min_{(s',a')\in B^{(k)}} \dist\big[(s^{(k)}_h, a^{(k)}_h), (s',a')\big].
	\]
	By definition of $Q^{(k)}_h(s^{(k)}_h, a^{(k)}_h)$, 
	we have
	\begin{align*}
	Q^{(k)}_h(s^{(k)}_h, a^{(k)}_h)
	&\le 	Q^{(k)}_h(s^{(k)*}_h, a^{(k)*}_h) +  L_1\cdot \dist[(s^{(k)*}_h, a^{(k)*}_h), (s^{(k)}_h, a^{(k)}_h)] \qquad\text{(*by Lemma~\ref{lem:q-lipschitz}*)}\\
	&\le 
	r(s^{(k)*}_h, a^{(k)*}_h) + \sup_{a''\in\cA}{Q}_{h+1}^{(k)}({f}(s^{(k)*}_h, a^{(k)*}_h), a'') + L_1\cdot \dist[(s^{(k)*}_h, a^{(k)*}_h), (s^{(k)}_h, a^{(k)}_h)]\\
	&\qquad\qquad\text{(*by definition of $Q^{(k)}_h$ in Line~\ref{line:qkh}*)}\\
	&\le 
	r(s^{(k)}_h, a^{(k)}_h) + \sup_{a''\in\cA}{Q}_{h+1}^{(k)}({f}(s^{(k)}_{h}, a^{(k)}_h), a'')  + (L_2+1)L_1\cdot \dist[(s^{(k)*}_h, a^{(k)*}_h), (s^{(k)}_h, a^{(k)}_h)]  \\
	&\qquad\qquad\text{(*by Lipshitz continuity of $Q_{h+1}^{(k)}$ and $r$*)}
	\\
	&\le 
	r(s^{(k)}_h, a^{(k)}_h) + 
	{Q}_{h+1}^{(k)}(s^{(k)}_{h+1}, a^{(k)}_{h+1})+ L\cdot b^{B^{(k)}}(s^{(k)}_h, a^{(k)}_h). 
	\end{align*}
\end{proof}

\begin{proof}[Proof of Theorem~\ref{thm:metric-regretbound}]
	From  Lemma 4, we have 
	\begin{align*}
	Q^*_h(s, a)\leq Q^{(k)}_h(s, a), \forall (s, a)\in\cS\times\cA, h\in[H], k\in[K]. 
	\end{align*}
	
	Denote the policy at episode $k$ as $\pi^{(k)}$.
	We can rewrite the the regret as 
	\begin{align*}
	\reg(K)&=\sum_{k=1}^K\Big[V^*_1(s_1^{(k)}) - \sum_{h=1}^{H}r(s^{(k)}_h, a_h^{(k)})\Big]
	\\
	&= \sum_{k=1}^K\Big[\max_{a\in\cA}Q_1^*(s_1^{(k)}, a) - Q^{\pi^{(k)}}_1(s_1^{(k)}, a_1^{(k)})\Big] \\
	&\leq \sum_{k=1}^K\Big[\max_{a\in\cA}Q_1^{(k)}(s_1^{(k)}, a) - Q^{\pi^{(k)}}_1(s_1^{(k)}, a_1^{(k)})\Big] \\
	&= \sum_{k=1}^K\Big[Q_1^{(k)}(s_1^{(k)}, a_1^{(k)}) - Q^{\pi^{(k)}}_1(s_1^{(k)}, a_1^{(k)})\Big]
	\end{align*}
	Next we consider $Q_h^{(k)}(s_h^{(k)}, a_h^{(k)}) - Q_h^{\pi^{(k)}}(s_h^{(k)}, a_h^{(k)})$.
	We have
	\begin{align*}
	Q_h^{(k)}&(s_h^{(k)}, a_h^{(k)}) - Q_h^{\pi^{(k)}}(s_h^{(k)}, a_h^{(k)})
	\qquad\text{(*by Lemma~\ref{lem:ind}*)}\\
	&\le L\cdot b^{B^{(k)}}(s_h^{(k)}, a_h^{(k)}) + {r}_{h+1}^{(k)}(s_h^{(k)}, a_h^{(k)}) + Q_{h+1}^{(k)}(s_{h+1}^{(k)}, a_{h+1}^{(k)}) - 
	r_{h+1}^{(k)}(s_h^{(k)}, a_h^{(k)})- Q_{h+1}^{\pi^{(k)}}(s_{h+1}^{(k)}, a_{h+1}^{(k)})\\
	&\le 
	 L\cdot b^{B^{(k)}}(s_h^{(k)}, a_h^{(k)})  +  Q_{h+1}^{(k)}(s_{h+1}^{(k)}, a_{h+1}^{(k)}) - 
	Q_{h+1}^{\pi^{(k)}}(s_{h+1}^{(k)}, a_{h+1}^{(k)})\\
	&\le 
	L\cdot b^{B^{(k)}}(s_h^{(k)}, a_h^{(k)}) + L\cdot b^{B^{(k)}}(s_{h+1}^{(k)}, a_{h+1}^{(k)}) +  Q_{h+2}^{(k)}(s_{h+2}^{(k)}, a_{h+2}^{(k)}) - 
	Q_{h+2}^{\pi^{(k)}}(s_{h+2}^{(k)}, a_{h+2}^{(k)})\\
	&\le \ldots\\
	&\le  L\cdot\sum_{h'=h}^{H}b^{B^{(k)}}(s_{h'}^{(k)}, a_{h'}^{(k)}).
	\end{align*}
	Moreover, we immediately have
	\[
	Q_1^{(k)}(s_1^{(k)}, a_1^{(k)}) - Q_1^{\pi^{(k)}}(s_1^{(k)}, a_1^{(k)})
	\le H.
	\]
	Therefore, 
	\begin{align*}
	\reg(K)\leq \sum_{k=1}^K\min\bigg\{L\cdot\sum_{h=1}
	^{H}b^{B^{(k)}}(s_h^{(k)}, a_h^{(k)}), ~H\bigg\}. 
	\end{align*}
	We consider an $\epsilon$-net, $\cN_{\epsilon}$, that covers $\cS\times \cA$.
	We now connect each $(s,a)$ to its nearest neighbor in $\cN(\epsilon)$.
	Denote
	\[
	\nn_{\epsilon}(s,a) = \arg\min_{(s',a')\in \cN(\epsilon)}
	\dist[(s,a), (s', a')].
	\] 	
	At episode $k\ge 1$,  if for some $k'< k$ and some $h'\in[H]$ there is $\nn_{\epsilon}(s_{h'}^{(k')}, a_{h'}^{(k')}) = (s,a)$, we call $(s,a)$ has been \emph{visited}.
	Thus if $\nn_{\epsilon}(s_{h}^{(k)}, a_{h}^{(k)}) = (s,a)$, 
	we can upper bound 
	\begin{align*}
		b^{B^{(k)}}(s_{h}^{(k)}, a_{h}^{(k)}) &\le 
		\dist[(s_{h}^{(k)}, a_{h}^{(k)}), (s_{h'}^{(k')}, a_{h'}^{(k')})]\le \dist[(s_{h}^{(k)}, a_{h}^{(k)}),(s,a) ] + \dist[(s_{h'}^{(k')}, a_{h'}^{(k')}), (s,a)]\\
		&\le 2\epsilon.
	\end{align*}
	On the other hand, if $\nn_{\epsilon}(s_{h}^{(k)}, a_{h}^{(k)})$ has not been visited, we upper bound the regret of the entire episode by $H$.
	However, such case can only happen at most $|\cN(\epsilon)|$ times as for the next episode, $\nn_{\epsilon}(s_{h}^{(k)}, a_{h}^{(k)})$ will become visited.
	Therefore,
	\[
	\reg(K)\leq H|\cN(\epsilon)| + 2\epsilon L KH
	\]
	as desired.

\end{proof}
%
We are now ready to prove Theorem~\ref{thm:corr-metric-regret}.
\begin{proof}[Proof of Theorem~\ref{thm:corr-metric-regret}]
	Since the metric space has a doubling dimension $d$, we can have an $\epsilon$-net $\cN(\epsilon)$ with size 
	\[
	|\cN(\epsilon)|= \Theta\Big(\frac{D}{\epsilon}\Big)^{d}.
	\]
	By Theorem~\ref{thm:metric-regretbound}, the regret is upper bounded by
	\[
	\reg(K)\le H\cdot\Theta\Big(\frac{D}{\epsilon}\Big)^{d} + 2\epsilon L K H
	\]
	When 
	\[
	\epsilon = D^{\frac{d}{d+1}}\cdot(LK)^{-\frac{1}{d+1}},
	\]
	we can upper bound the regret as
	\begin{align*}
	&\reg(K)=O(DLK)^{\frac{d}{d+1}}\cdot H.
	\end{align*}
	as desired.
\end{proof}

\subsection{Optimality}
Next we establish a regret lower bound for reinforcement learning in deterministic metric MDP.
\begin{theorem}[{\bf Minimax Lower Bound}]
	\label{thm:regretlowerbound}
	Let $ \cM(H, \epsilon)$ be a family of MDPs with the form  $M=(\cS, \cA, f, r, H)$, where 
	$\cX:= \cS\times \cA$ is a metric space that admits an $\epsilon$-packing $\cC(\epsilon)$ for some $\epsilon>H/2$, and $M$ satisfies Assumption~\ref{assump-metric}.
	Let $\cK$ be any online algorithm that admits input any MDP from $M$.
	Let $\reg_{\cK}^{M}(K)$ denote the regret for $\cK$ on $M$ after $K\ge 1$ episodes.
	Then 
	\begin{align*}
	\max_{M\in \cM(H,\epsilon)} &\reg_{\cK}^{M}(K)\geq \Omega\Big[
	\min\big(\big|\cC(\epsilon)\big|, K\big)\cdot H
	\Big]
	\end{align*}
\end{theorem}
The proof is postponed to the appendix. 
The core idea of proving the theorem is to construct a hard instance distribution such that an MDP sampled from the distribution satisfies:
\begin{itemize}
	\item every two distinct state-action pairs have distance exactly $H$ so that the Lipshitz continuity conditions in Assumption~\ref{assump-metric} are always satisfied;
	\item has absorbing states so that any algorithm can explore at most one state-action pair per episode;
	\item has only one random non-absorbing state-action pair with reward $1$  and others with reward $0$.
\end{itemize} 
Since the rewarding state-action pair is random, any algorithm is expected to spend $\Theta(|\cC(\epsilon)|)$ episodes until it reaches the rewarding state-action pair.
But an ``oracle'' optimal algorithm can pick the  rewarding state-action pair for every episode.
Therefore, any exploration-based algorithm requires to pay a regret $\Omega( |\cC(\epsilon)|H)$ in this hard instance distribution.


\section{Examples and Extensions}

Our method and the analysis apply to several important special cases. 

\subsection{Finite State-Action MDP}

In the case of finitely many states and actions without any structural knowledge, one can simply pick the metric to be
$$\dist((s,a), (s',a')) =  H,\qquad\forall (s,a)\neq (s',a') .$$
So we can see Algorithm 2 contains the basic Algorithm 1 as a special case.
For discrete space, if we take $\epsilon = H-\delta$ for an arbitrary $\delta >0$, then the covering size is $N(\epsilon) = SA$. 
There is no other conver needed to be considered.
Then Theorem \ref{thm:metric-regretbound} implies that
the $K$-episode regret is  $SAH$ regardless of $H$, which matches Theorem \ref{thm:opt-disc-reg}.

\subsection{Linear Model with Feature Map}

An important family of structured MDP is the family where the reward and transition $r,f$ are linear with respect to some feature map $\phi(s,a) \in \RR^d$. In this case, $Q^{\pi}_h$ and $Q^*_h$ are all linear in the feature space. 

Let us adapt Algorithm 2 to work with the linear model. To do so, we use a different function approximator to capture the class of linear $Q$ functions.
Given a data set $\{x_i,y_i\}_{i=1}^N$, we let $\Phi_X$ be the $N\times d$ matrix whose $i$-th row is $\phi(x)^T$ and let $\mathbf{y}\in \RR^N$ be the vector whose $i$-th entry is $y_i$. We use the following function approximation oracle
\begin{align*}
&\funcapprox_{\{(x_i,y_i)\}_{i}}(x)= \qquad\left\{
\begin{tabular}{c c}
$\phi(x)^\top(\Phi_X^\top\Phi_X)^{-1} \Phi_X^\top \mathbf{y}$ &  $\phi(x)\in \hbox{Span}(\{\phi(x_i)\})$\\
&\vspace{-3mm}\\
H & $\phi(x)\notin \hbox{Span}(\{\phi(x_i)\})$
\end{tabular}
\right.
\end{align*}
The above approximator fits a linear function on the observed $\{(x_i,y_i)\}$. When it is queried at a point $x$ that does not belong to the subspace spanned by the observations $x_i$'s, it will output an upper bound $H$.

Then we can show that the regret depends linearly on the feature dimension:
$$\reg(K) \leq H d.$$
The proof follows similarly as that of Theorem 2: the instant regret when visiting $(s,a)$ is bounded by $b(s,a)$, which is zero if $(s,a)$ is the linear combination of seen states in the feature space and equals to $ H$ whenever $\phi(s,a) \notin \hbox{Span}(\{\phi(s_i,a_i)\}_{i\in B}) .$ This would happen at most $d$ times since the feature space has dimension $d$. 

This result matches the optimal regret bound established in \cite{wen2017efficient}.
It is a $O(1)$ regret that does not depend on the episode number $K$. It scales linearly (instead of exponentially) with respect to dimension of the feature space.

\section{Conclusion}

This paper provides a simple upper-confidence reinforcement learning algorithm for episodic deterministic system. Given a metric over the state-action space that captures continuity of the rewards and transition functions, the algorithm achieves sublinear regret that depends on the doubling dimension of the state-action space. We show that this regret is non-improvable in general. Our method can be adapted to achieve the state-of-art $O(1)$ regret in the setting where the value functions can be represented by a linear combination of features.

\newpage

%

\bibliographystyle{apalike}
\bibliography{refregret}

\onecolumn
\appendix

\section{Missing Proofs}
\begin{proof}[Proof of Theorem~\ref{thm:opt-disc-reg}]
	Note that $Q^*_h\le H\cdot\mathbf{1}$ for any $h\in [H]$.
	Denote $\wh{r}^{(k)}$ as the $\wh{r}$ at the beginning of the $k$th episode.
	Similarly we denote $\wh{f}^{(k)}$ and $b^{(k)}$.
	By definition of the algorithm, we have
	\[
	\forall k\in [K]:\quad Q_h^*= r\le \wh{r}^{(k)} := Q^{(k)}_H.
	\]
	We can therefore show inductively that
	\[
	\forall k\in [K], (s,a)\in \cS\times \cA:\quad 
	Q^*_h[s, a]\le \min\big(H, \wh{r}^{(k)}(s, a) + \max_{a'\in \cA}Q^{(k)}_{h+1}[\wh{f}^{(k)}(s,a), a']  + b^{(k)}(s,a)\big) := Q^{(k)}_h[s, a].
	\]
	Denote $s_h^{(k)}$ as the state at time $(k,h)$
	and $a_h^{(k)} = \pi^{(k)}(s_h^{(k)}, h) =\arg\max_{a\in \cA} Q^{(k)}_h(s,a)$.
	Denote the policy at episode $k$ as $\pi^{(k)}$.
	We can rewrite the the regret as 
	\[
	\reg(K):=\sum_{k=1}^K\Big[V^*_1[s_0] - \sum_{h=1}^{H}r(s^{(k)}_h, a_h^{(k)})\Big]
	= \sum_{k=1}^K\Big[V^*_1[s_0] - V^{\pi^{(k)}}_1(s_0)\Big].
	\]
	Consider $V^*_h(s) - V^{\pi^{(k)}}_h(s)$.
	Denote 
	\[
	\forall s\in \cS:\quad V^{(k)}_h(s) := \max_{a\in \cA} Q^{(k)}_h(s,a).
	\] 
	Hence  \[
	\forall s\in \cS:\quad V^*(s)=\max_{a'}Q^*(s,a)\le V^{(k)}_h(s).
	\]
	We can thus upper bound $V^*_h(s) - V^{\pi^{(k)}}_h(s)$ as follows.
	\[
	\forall s\in \cS, h\in [H]:\quad V^*_h(s) - V^{\pi^{(k)}}_h(s)
	\le V^{(k)}_h(s) - V^{\pi^{(k)}}_h(s).
	\]
	Note that for all $h\in [H-1], s\in \cS$, we have
	\begin{align*}
	V^{(k)}_h[s_h^{(k)}] &= \min\{H, \quad \wh{r}^{(k)}(s_h^{(k)},  a_h^{(k)}) +  V^{(k)}_{h+1}[\wh{f}^{(k)}(s_h^{(k)}, a_h^{(k)})] + b^{(k)}(s_h^{(k)}, a_h^{(k)})\}\\
	& \le \min\{H, \quad \wh{r}^{(k)}(s_h^{(k)},  a_h^{(k)}) +  V^{(k)}_{h+1}(s_{h+1}^{(k)}) + b^{(k)}(s_h^{(k)}, a_h^{(k)})\}.
	\end{align*}
	Since
	\[
	V^{\pi^{(k)}}_h[s_h^{(k)}] =
	{r}^{(k)}(s_h^{(k)}, a_h^{(k)}) +   
	V^{\pi^{(k)}}_{h+1}(s_{h+1}^{(k)}).
	\]
	Thus
	\begin{align*}
	V^{(k)}(s_h^{(k)}, h) - V^{\pi^{(k)}}(s_h^{(k)}, h) 
	&\le  \wh{r}^{(k)}(s_h^{(k)}, a_h^{(k)}) - {r}^{(k)}(s_h^{(k)}, a_h^{(k)}) +b^{(k)}(s_h^{(k)}, a_h^{(k)})\\
	& \quad +   V^{(k)}_{h+1}(s_{h+1}^{(k)}) - V^{\pi^{(k)}}_{h+1}(s_{h+1}^{(k)}).
	\end{align*}
	Denote 
	\[
	\wh{b}^{(k)}(s_h^{(k)}, a_h^{(k)}) := \wh{r}^{(k)}(s_h^{(k)}, a_h^{(k)}) - {r}^{(k)}(s_h^{(k)}, a_h^{(k)}) +b^{(k)}(s_h^{(k)}, a_h^{(k)}).
	\]
	Now we can recursively bound 
	\begin{align*}
	V^{(k)}_1(s_0^{(k)}) - V^{\pi^{(k)}}_1(s_0^{(k)})
	&\le V^{(k)}_2(s_1^{(k)}) - V^{\pi^{(k)}}_2(s_1^{(k)})
	+ \wh{b}^{(k)}(s_1^{(k)}, a_1^{(k)})
	\\
	&\le  \ldots\\
	&\le V^{(k)}_h(s_h^{(k)}) - V^{\pi^{(k)}}_h(s_h^{(k)})
	+ \sum_{h'=1}^{h-1}\wh{b}^{(k)}(s_{h'}^{(k)}, a_{h'}^{(k)}) \\
	&\le \sum_{h'=1}^{H}\wh{b}^{(k)}(s_{h'}^{(k)}, a_{h'}^{(k)})
	\end{align*}
	where we denote $V^{(k)}_{H+1}(\cdot) = V^{\pi^{(k)}}_{H+1}(\cdot) = 0$.
	Moreover, we can immediately bound
	\[
		V^{(k)}_1(s_0^{(k)}) - V^{\pi^{(k)}}_1(s_0^{(k)})\le H.
	\]
	Therefore,
	\[
	\reg(K)\le \sum_{k=1}^K\min\Big[H, ~\sum_{h'=1}^{H}\wh{b}^{(k)}(s_{h'}^{(k)},a_{h'}^{(k)})\Big].
	\]
	It remains to bound
	\[
	\sum_{h'=1}^{H}\wh{b}^{(k)}(s_{h'}^{(k)}, a_{h'}^{(k)}).
	\]
	For each $(s_{h}^{(k)}, a_{h}^{(k)})$, if it has been visited in the past $k-1$ episodes, then
	\[
	\wh{b}^{(k)}(s_{h}^{(k)}, a_{h}^{(k)}) = \wh{r}^{(k)}(s_h^{(k)}, a_h^{(k)}) - {r}^{(k)}(s_h^{(k)}, a_h^{(k)}) +b^{(k)}(s_h^{(k)}, a_h^{(k)}) = 0.
	\]
	If it is visited for first time, 
	then 
	\[
	\wh{b}^{(k)}(s_{h}^{(k)}, a_{h}^{(k)}) = \wh{r}^{(k)}(s_h^{(k)}, a_h^{(k)}) - {r}^{(k)}(s_h^{(k)}, a_h^{(k)}) +b^{(k)}(s_h^{(k)}, a_h^{(k)}) \le H+1.
	\]
	If there exists such a $(s_{h}^{(k)}, a_{h}^{(k)})$, then the regret of the entire episode can be bounded by $H$.
	Since there are $SA$ number of such $(s,a)$ pairs, we have
	\[
	\sum_{k=1}^K\Big[V^*_1[s_0] - V^{\pi^{(k)}}_1(s_0)\Big]
	\le \sum_{k=1}^K H\cdot
	\mathbb{I}(\text{there exists an $h$, s.t. $(s_h^{(k)}, a_h^{(k)})$ is visited the first time})\le SAH.
	\]
\end{proof}

\begin{theorem}
	\label{thm:reg-lb-det}
	Denote $\cM(\cS,\cA, H)$ to be the set of all deterministic MDPs with states $\cS$, actions $\cA$ and horizon $H$.
	Let $\cK$ be an online algorithm for $\cM(\cS,\cA, H)$.
	Then
	\[
\max_{M\in \cM(\cS,\cA, H)} \reg_{\cK}^M(K) =\Omega(\min(|\cS||\cA|, K) H)
	\]
	as long as  $H\ge \log|\cS|$.
\end{theorem}

\begin{proof}[Proof of Theorem~\ref{thm:reg-lb-det}]
	We construct a distribution $\mu$ on $\cM(\cS,\cA, H)$ as follows.
	All the MDPs in the support of $\mu$ start with a binary tree $T$ rooted at state $s_0$. For every state in the tree, action $a_0\in \cA$ takes the transition to the left child and all other actions take the transition to the right child.
	The binary tree has $O(\log(|\cS|))$ layers. 
	Each edge of the tree has $0$ reward.
	Denote the leaves of the tree as $\cS'\subset \cS$. 
	Without loss of generality, we let the size of the tree to be $|\cS|-2$ and hence $|\cS'|=\lceil(|\cS|-2)/2 \rceil$. 
	We denote two absorbing states $s_n, s_r\in \cS$ that are not in the tree. 
	Every action on $s_r, s_n$ self-loops. 
	$s_n$ generates no reward for all actions.
	$s_r$ generates reward $1$ for every action.
	We pick a random $(s^*,a^*)\in \cS'\times \cA$, and set $f(s^*,a^*)=s_r$.
	For the rest $(s,a)$, we set $f(s,a) = s_n$.
	Note that an optimal policy will reach $(s^*,a^*)$ and obtain reward $H-O(\log(|\cS|))$.
	
	Next, we show that for any deterministic algorithm $\cK_{det}$, the expected regret on the distribution $\mu$ is $\Omega(|\cS||\cA|H)$.
	Note that for every MDP $M\in \supp(\mu)$ has the same initial structures. 
	For a deterministic algorithm $\cK_{det}$, we consider a particular instance $\wt{M}\in \cM(\cS, \cA, H)$:
	the initial structure of $\wt{M}$ is $T$, but very state-action pair $(s,a)\in\cS'\times\cA$ transitions to $s_n$.
	Consider $\cK_{det}$ runs on $\wt{M}$.
	Suppose we have run $\cK_{det}$ for $K=p|\cS||\cA|$ episodes with $p<1$.
	Denote the state-action pair reached at the end of episode $k\in [K]$ as $\big(\wt{s}'(\cK_{det}, k), \wt{a}'(\cK_{det}, k)\big)$.
	
	Suppose we now run $\cK_{det}$ on an instance sampled from $\mu$, then our claim is that with probability at least $1-p$, $\cK_{det}$ has payed regret at least $K(H-O(\log(|\cS|)))$.
	Indeed, for an instance $M\sim \mu$,
	if for all $k$, $\big(\wt{s}'(\cK_{det}, k), \wt{a}'(\cK_{det}, k)\big)$ on $M$ does not equal to $(s^*, a^*)$, which happens with probability $1-K/(|\cS||\cA|) = 1-p$, then $\cK_{det}$ would have the exact same history on $M$ as it runs on $\wt{M}$.
	Therefore, it pays regret $K(H-O(\log(|\cS|)))$ on $M$.
	Hence,
	\begin{align*}
	\min_{\cK_{det}}\EE_{M\sim\mu}\big[\reg_{\cK_{det}}^M(K)\big]&=(1-p)\cdot K(H-O(\log(|\cS|)))\\
	&=\Omega(|\cS||\cA|H).
	\end{align*}
	as long as $K=\Omega(|\cS||\cA|)$.
	Denote $\nu$ as an distribution on $\cM(\cS,\cA, H)$, then we have,
	\[
	\sup_{\nu}\min_{\cK_{det}}\EE_{M\sim\nu}\big[\reg_{\cK_{det}}^M(K)\big] \ge \min_{\cK_{det}}\EE_{M\sim\mu}\big[\reg_{\cK_{det}}^M(K)\big] = \Omega(|\cS||\cA|H).
	\]
	By Yao's minimax \cite{yao1977probabilistic} theorem, we have,
	\[
	\min_{\cK}\max_{M\in \cM(\cS, \cA, H)}\big[\reg_{\cK}^M(K)\big] \ge  \sup_{\nu}\min_{\cK_{det}}\EE_{M\sim\nu}\big[\reg_{\cK_{det}}^M(K)\big]  = \Omega(|\cS||\cA|H).
	\]
	This completes the proof.
\end{proof}

\begin{proof}[Proof of Theorem~\ref{thm:regretlowerbound}]
	We will use the same distribution as in the proof of Theorem~\ref{thm:reg-lb-det} to prove the theorem.
	Note that the space $\cS\times \cA$ in the proof of Theorem~\ref{thm:reg-lb-det} is not a metric space yet.
	To convert it to  a metric space, we assign a na\"ive metric by setting
	\[
	\dist[(s,a), (s',a')] = H\cdot \mathbb{I}[(s,a)=(s',a')].
	\]
	Since the optimal action-value function $Q^*$ is upper bounded by $H$ uniformly, the Lipschitz continuity conditions in Assumption~\ref{assump-metric} can be verified for any $L_1\ge1$ and $L_2\ge 1$.
	Then Theorem~\ref{thm:regretlowerbound} is proved the same way as Theorem~\ref{thm:reg-lb-det}.
\end{proof}

\end{document}